\documentclass[letterpaper]{article}

\usepackage{natbib,alifeconf}  

\usepackage{amsmath}
\usepackage{amsthm}
\usepackage{thmtools}
\usepackage{amsfonts}
\usepackage{amssymb}
\usepackage{mathtools}

\usepackage{url,hyperref}
\usepackage{cleveref}
\usepackage{booktabs}

\usepackage{microtype}

\usepackage[T1]{fontenc}

\title{A ``Good Regulator Theorem'' for Embodied Agents}

\author{
    Nathaniel Virgo$^{1,2}$,
    Martin Biehl$^{3}$,
    Manuel Baltieri$^{4}$,
    \and
    Matteo Capucci$^{5}$\\
    $^1$University of Hertfordshire, UK 
    \hspace{1em}
    $^2$Earth-Life Science Institute (ELSI), Institute of Science Tokyo, Japan
    \\
    $^3$Cross Labs, Cross Compass Ltd., Japan
    \hspace{1em}
    $^4$Araya, Inc., Japan
    \hspace{1em}
    $^5$Independent researcher
} 

\newtheorem{theorem}{Theorem}[section]

\newtheorem{lemma}[theorem]{Lemma}

\theoremstyle{definition}
\newtheorem{definition}[theorem]{Definition}

\usepackage{silence}
\usepackage{caption}
\usepackage{subcaption}

\usepackage{quiver}
\usepackage{outlines}

\usepackage{globalvals}
\newcommand{\defFootnote}{\defVal}
\newcommand{\insFootnote}[1]{\footnote{\useVal{#1}\label{#1}}}
\usepackage[normalem]{ulem}

\newcommand{\pow}{\mathcal{P}}

\defcitealias{conant1970}{C\&A}
\newcommand{\CA}{\citepalias{conant1970}}

\newcommand{\CAauthor}{\citetalias{conant1970}}

\usepackage{needspace}
\newcommand{\myparagraph}[1]{%
  \vspace{1mm}
  \needspace{2\baselineskip}%
  \noindent\textbf{#1.}}

\makeatletter
\let\save@mathaccent\mathaccent
\newcommand*\if@single[3]{%
  \setbox0\hbox{${\mathaccent"0362{#1}}^H$}%
  \setbox2\hbox{${\mathaccent"0362{\kern0pt#1}}^H$}%
  \ifdim\ht0=\ht2 #3\else #2\fi
  }
\newcommand*\rel@kern[1]{\kern#1\dimexpr\macc@kerna}
\newcommand*\widebar[1]{\@ifnextchar^{{\wide@bar{#1}{0}}}{\wide@bar{#1}{1}}}
\newcommand*\wide@bar[2]{\if@single{#1}{\wide@bar@{#1}{#2}{1}}{\wide@bar@{#1}{#2}{2}}}
\newcommand*\wide@bar@[3]{%
  \begingroup
  \def\mathaccent##1##2{%
    \let\mathaccent\save@mathaccent
    \if#32 \let\macc@nucleus\first@char \fi
    \setbox\z@\hbox{$\macc@style{\macc@nucleus}_{}$}%
    \setbox\tw@\hbox{$\macc@style{\macc@nucleus}{}_{}$}%
    \dimen@\wd\tw@
    \advance\dimen@-\wd\z@
    \divide\dimen@ 3
    \@tempdima\wd\tw@
    \advance\@tempdima-\scriptspace
    \divide\@tempdima 10
    \advance\dimen@-\@tempdima
    \ifdim\dimen@>\z@ \dimen@0pt\fi
    \rel@kern{0.6}\kern-\dimen@
    \if#31
      \overline{\rel@kern{-0.6}\kern\dimen@\macc@nucleus\rel@kern{0.4}\kern\dimen@}%
      \advance\dimen@0.4\dimexpr\macc@kerna
      \let\final@kern#2%
      \ifdim\dimen@<\z@ \let\final@kern1\fi
      \if\final@kern1 \kern-\dimen@\fi
    \else
      \overline{\rel@kern{-0.6}\kern\dimen@#1}%
    \fi
  }%
  \macc@depth\@ne
  \let\math@bgroup\@empty \let\math@egroup\macc@set@skewchar
  \mathsurround\z@ \frozen@everymath{\mathgroup\macc@group\relax}%
  \macc@set@skewchar\relax
  \let\mathaccentV\macc@nested@a
  \if#31
    \macc@nested@a\relax111{#1}%
  \else
    \def\gobble@till@marker##1\endmarker{}%
    \futurelet\first@char\gobble@till@marker#1\endmarker
    \ifcat\noexpand\first@char A\else
      \def\first@char{}%
    \fi
    \macc@nested@a\relax111{\first@char}%
  \fi
  \endgroup
}
\makeatother

\begin{document}
\frenchspacing

\maketitle

\begin{abstract}

In a classic paper, Conant and Ashby claimed that ``every good regulator of a system must be a model of that system.''
Artificial Life has produced many examples of systems that perform tasks with apparently no model in sight; these suggest Conant and Ashby's theorem doesn't easily generalise beyond its restricted setup.
Nevertheless, here we show that a similar intuition can be fleshed out in a different way: whenever an agent is able to perform a regulation task, it is possible for an observer to interpret it as having ``beliefs'' about its environment, which it ``updates'' in response to sensory input.
This notion of belief updating provides a notion of model that is more sophisticated than Conant and Ashby's, as well as a theorem that is more broadly applicable.
However, it necessitates a change in perspective, in that the observer plays an essential role in the theory: models are not a mere property of the system but are imposed on it from outside.
Our theorem holds regardless of whether the system is regulating its environment in a classic control theory setup, or whether it's regulating its own internal state; the model is of its environment either way.
The model might be trivial, however, and this is how the apparent counterexamples are resolved.

\end{abstract}

\section{Introduction}

The work of W.\ Ross Ashby had a profound impact on the development of Artificial Life, and of cognitive science more broadly.
In particular, \cite{ashby1960design} led to the idea of the \emph{viability boundary} of an organism, a notion of current interest in ALIFE \citep{mcshaffrey2023decomposing,beer2024viability,egbert2018methods} and a foundational concept in enactivism and the evolutionary robotics approach to cognition \citep[e.g.][]{Beer1995,beer1997dynamics,beer2004autopoiesis,DiPaolo2005adaptivity}.

In parallel, \citep{conant1970} claimed in its title that ``every good regulator of a system must be a model of that system'', an idea that was taken up in control theory under the name `internal model principle' \citep{francis1976internal} and whose influence on cognitive science can still be felt today, for example, in some works under the banner of the free energy principle (FEP) such as \citep{friston2019free,daCosta2021bayesian}, or in the information-theoretic definition of `semantic information' \citep{kolchinsky2018semantic}.

Although the ideas in \citep{ashby1960design} and \citep[hereafter \CAauthor{}]{conant1970} are not dissimilar, the bodies of work descending from them appear to have reached disparate conclusions: Artificial Life has produced many examples of agents that apparently do not need an internal model in order to achieve their tasks \citep[e.g.][]{braitenberg1986vehicles,beer2003dynamics}, apparently contradicting \CA{}.

This is in part because \CA{} doesn't strictly succeed, even in its own terms, in showing what its title claims. (This has been noted previously in blog posts: \citealp{baez2016blog,wentworth2021blog}.)
Its theorem doesn't establish that \emph{every} good regulator is a model but only that some of them are (those that are not ``unnecessarily complex'' in \CAauthor{}'s terminology).
Although there have been several attempts to generalise it beyond its simple setup, they often seem to need strong assumptions that limit their validity. (See \citealp{baltieri2025imp} for a recent exposition of the assumptions behind the internal model principle in control theory, for example.)

We claim, however, that these issues can be fixed in a way that removes the need for extra assumptions while also accounting for the apparent counterexamples.
To do this, we need to use a different notion of `model' than \CAauthor{}'s one.

For \citeauthor{conant1970}, for a system ${X}$ to \emph{be a model} of a system ${Y}$ there should be a ``\,`-morphic' relation'', i.e.\ a {homomorphism} of some kind from, ${Y}$ to ${X}$.
This means roughly that the model ${X}$ behaves like a coarse-grained version of the thing it's modelling, ${Y},$ with an intuition along the lines that a controller must behave like the system it's controlling in order to predict it.

In contrast, our notion of model builds on a recent notion of \emph{interpretation map} (\citealp{virgo2021interpreting,biehl_interpreting_2023}. The idea is that it must be possible for an observer to attribute belief states (priors) to the states of $X$ that change consistently with the dynamics of model $Y$. In the terminology of \cite{seth2018being} this is a version of ``having a model'' (as opposed to ``being a model'') in that it is about states encoding priors, although see the comments on `as-if' agency below. In the current work we consider ``possibilistic'' belief states, as in \citep[section III]{baltieri2025imp}, which are simpler than Bayesian priors, though closely analogous.

We show that this version of `has a model' has its own ``good regulator theorem'' (\cref{main-theorem.thm}), which doesn't need any assumptions beyond a fairly minimal definition of regulation.
Although this definition of regulation is minimal, it is flexible enough to apply to embodied agents.
The theorem shows that every good regulator has a model, but it allows that the model might be trivial, and this is how the apparent counterexamples are accounted for.

A related result is proved in \citep{baltieri2025imp} in the context of the Internal Model Principle (IMP).
Here we take a more first-principles approach, which doesn't need the assumptions of the IMP and leads to a simpler yet stronger theorem.
In particular, \citep{baltieri2025imp} requires that the controller can fully observe the plant, which limits the extent to which it can be applied to embodied agents.
Nevertheless, this prior work can be seen as showing that ``being a model'' implies ``having a model'' in senses similar to ours.

The study of interpretation maps builds on a research programme of `as-if' agency \citep{mcgregor2016asif,mcgregor2017bayesian,mcgregor2025formalising1,mcgregor2025formalising2}, which aims to formalise Dennett's ideas about the intentional stance \citep[e.g.][]{dennett2006systems}.
The idea of a \emph{stance} is central to our work: our notion of model depends fundamentally on choices that an external observer must make about how to interpret the system's internal state and behaviour.
A model is part of a stance the observer can optionally take with respect to the system, rather than a property of the system alone.
(In this respect our work could be characterised as ``is a model'' rather than ``has a model'' in \citeauthor{seth2018being}'s 
terms, but we will continue to use the ``has a model'' terminology.)
At the same time, the choice is not arbitrary: %
the patterns in a system's dynamics that make an interpretation possible are a ``real'' property of that system
 \citep[c/f][]{dennett1991-real-patterns}.

Our work could be seen as laying theoretical groundwork for similar ideas in machine learning, for example in work focusing on extracting Bayesian models \citep{ortega2019metalearningsequentialstrategies} or world-models \citep{richens2025generalagentsneedworld} from agents trained on multiple tasks.
There are also has some similarities to recent work on emergence \citep{rosas2024software}, and we suspect the two approaches will turn out to support each other, since our models are probably best thought of as interpretations of the macroscopic ``software'' level rather than the micro.
In the conclusion section we also discuss potential relations to the free energy principle (FEP).

Our notion of regulation includes both `extrinsic' cases such as a thermostat regulating the temperature of a room and `intrinsic' cases such as an organism regulating its own body temperature.
Because of this, we can claim as an informal corollary that \emph{every system that is a good regulator of itself, must have a model of its environment.}
This is perhaps surprising, and suggests a possible connection to notions of {sense-making}  in enactive cognitive science, where meaning is said to emerge from an organism's need to maintain itself as a physical metabolic system \citep{DiPaolo2005adaptivity,DeJaegher2007participatory}.
The term ``model'' might seem counter to enactivism, but our approach is 
closer to a dynamical systems approach than a traditionally computational one.
Our models arise from the dynamics rather than being invoked to explain them, and in a biological context they might have more in common with the enactive notion of meaning than with a traditional computationalist notion of model.

\section{Coupled systems and regulation}

A common framework for modelling embodied agents is the \emph{sensorimotor loop} \citep{klyubin_organization_2004,bertschinger_information_2006,zahedi_higher_2010}.
The precise details can differ, but one typically assumes, for simplicity, that time proceeds in discrete steps and that the states of the agent and its environment change over time according to a causal Bayesian network along the lines of the following:
\begin{equation}
    \label{sensorimotor-loop-dag}
    \begin{tikzpicture}[
            baseline = (current bounding box.center),
            rv/.style={circle,draw,
                inner sep=0pt,minimum size=1.7em},
            x=1.6em,
            y=1.5em,
        ]
        \node[rv] at (0,0) ({X}1) {${X}_1$};
        \node[rv] at (1,1) ({A}1) {${A}_1$};
        \draw[->] ({X}1) -- ({A}1);
        \node[rv] at (2,2) ({Y}1) {${Y}_1$};
        \draw[->] ({A}1) -- ({Y}1);
        \node[rv] at (3,1) ({S}1) {${S}_1$};
        \draw[->] ({Y}1) -- ({S}1);
        \node[rv] at (4,0) ({X}2) {${X}_2$};
        \draw[->] ({S}1) -- ({X}2);
        \draw[->] ({X}1) -- ({X}2);
        
        \node[rv] at (5,1) ({A}2) {${A}_2$};
        \draw[->] ({X}2) -- ({A}2);
        \node[rv] at (6,2) ({Y}2) {${Y}_2$};
        \draw[->] ({A}2) -- ({Y}2);
        \node[rv] at (7,1) ({S}2) {${S}_2$};
        \draw[->] ({Y}2) -- ({S}2);
        \node[rv] at (8,0) ({X}3) {${X}_3$};
        \draw[->] ({S}2) -- ({X}3);
        \draw[->] ({X}2) -- ({X}3);
        \draw[->] ({Y}1) -- ({Y}2);

        \node[rv] at (9,1) ({A}3) {${A}_3$};
        \draw[->] ({X}3) -- ({A}3);
        \node[rv] at (10,2) ({Y}3) {${Y}_3$};
        \draw[->] ({A}3) -- ({Y}3);
        \draw[->] ({Y}2) -- ({Y}3);

        \node[overlay] at (13,2) ({Y}n) {$\dots$};
        \draw[->] ({Y}3) -- ({Y}n);
        \node[overlay] at (11,0) ({X}n) {$\dots$};
        \draw[->] ({X}3) -- ({X}n);
    \end{tikzpicture}
\end{equation}
Here, ${X}_1, {X}_2, \dots$ are the agent's state at different times, ${Y}_1, {Y}_2, \dots$ the state of its environment (which may be taken to include the agent's position within it), ${A}_1, {A}_2, \dots$ the actions it takes on each time step and ${S}_1, {S}_2, \dots$ the inputs it gets to its sensors.

We will formalise the sensorimotor loop in a way that's close to this in spirit but different in look-and-feel, based on the concept of a `machine'.
Although the language of machines might feel computational in nature, machines are really just a formalism for talking about coupled dynamical systems in discrete time.
One should think of them as modelling an embodied agent coupled to its environment, in the same way that coupled continuous-time dynamical systems are used in \citep{Beer1995,beer1997dynamics} for example.

We restrict our attention to deterministic systems in discrete time, but the concept of machine can be extended to the stochastic case and the time-continuous case, among others---see \citep{myers2023categorical} for a common framework. 
Our results are somewhat specialised to the discrete time non-stochastic case.
Their extension to the stochastic case is not obvious, though it would be straightforward to generalise them to the possibilistic machines used in \citep{baltieri2025imp}.

We define two different types of machine, for reasons that will become apparent shortly.
We first fix sets ${A}$ and ${S}$, which are the set of actions the agent can take on each time step, and the set of states its sensors can take; together we call them the \emph{interface}.

\begin{definition}
    A \emph{Moore machine} consists of a \emph{state space}, which is a set ${X}$, together with a \emph{readout function} ${r}:{X}\to {A}$ and an \emph{update function} ${u}:{X}\times {S}\to {X}$.
\end{definition}
We will use a Moore machine to model the agent part of the sensorimotor loop.
The idea is that the map ${r}$ determines the action the agent takes, as a function of its state, while the map ${u}$ determines how the agent's internal state updates, as a function of its previous state and its sensor input.
Although we use the language of action selection and state updating, these should really just be thought of in dynamical 
terms; they are nothing but a description of how the system couples to its environment and changes over time.
We will think of readout as happening \emph{before} the update, so that on each time step the agent first takes an action and then receives a sensor input, at which point its state changes.

We model the environment with a different kind of machine, called a Mealy machine.
We define it with its inputs and outputs swapped compared to a Moore machine, since it takes the agent's action as an input and produces a sensor value as an output.
Its definition is slightly simpler than that of a Moore machine, in that it only has one function.
\begin{definition}
    A \emph{Mealy machine} consists of a \emph{state space}, which is a set ${Y}$, together with an \emph{evolution function} ${{e}:{Y}\times {A}\to {Y}\times {S}}$.
\end{definition}

The basic idea of the sensorimotor loop is that when an agent is coupled to an environment, they form a (closed) dynamical system.
Since we are dealing with discrete time and discrete state spaces, for us a ``dynamical system'' is just a set ${W}$ called the $\emph{state space}$, together with a function ${h}:{W}\to {W}$ called the \emph{update map}.
The update map ${h}$ takes the state at the current time and returns the next state.

We can now define what it means to couple an agent to its environment:
\begin{definition}
    \label{coupled-system.def}
    Given a Moore machine $({X},{{r}:{X}\to {A}},$ ${{u}:{X}\times{S}\to{X}})$ and a Mealy machine $({Y},{e}:{Y}\times {S}\to {Y}\times {A})$, the corresponding \emph{coupled system} is a dynamical system with state space ${X}\times {Y}$ and update map ${h}:{X}\times {Y}\to {X}\times {Y}$ given by
    \begin{equation}
        \label{coupled-system.eqn}
        {h}({x},{y}) = ({u}({x},{s})\,,\,{y}'),
    \end{equation}
    where ${y}'$ and ${s}$ are defined by $({y}',{s}) = {e}({y},{r}({x}))$.
\end{definition}
In the formal language of \emph{string diagrams} \citep[see e.g.][]{baez2010physics,selinger2010survey,coecke2017picturing}, \Cref{coupled-system.eqn} can be represented as
\begin{equation}
    \label{stringdiagram.eqn}
    h
    \,\,=\,\,
    \begin{tikzpicture}[
            baseline=(current bounding box.center),
            box/.style={
                draw,
                fill=white,
                inner sep=0pt,
                minimum size=1.5em
            },
            inout/.style={
                fill=white,
                inner sep=0pt,
                minimum size=1.5em
            },
            x=3.5em,
            y=1.3em,
            nudge/.store in=\nudge,
            nudge=0.35em,
        ]
        \draw [yshift=-\nudge] (0,0) -- (5,0);
        \draw [yshift=\nudge] (0,2) -- (5,2);
        \node[box] at (1.8,1) ({r}) {${r}$};
        \node[box] at (2.8,2) ({e}) {${e}$};
        \node[box] at (3.8,0) ({u}) {${u}$};
        \node[circle, fill=black, minimum size=1.5mm, inner sep=0,yshift=-\nudge] at (1,0) (copy) {};
        \draw (copy) to[out=60,in=180] ({r});
        \draw ({r}) to[out=0,in=180] ([yshift=-\nudge]{e}.west);
        \draw ([yshift=-\nudge]{e}.east) to[out=0,in=180] ([yshift=\nudge]{u}.west);
        \node[inout,yshift=-\nudge] at (0,0) {${X}$};
        \node[inout,yshift=-\nudge] at (5,0) {${X}$};
        \node[inout,yshift=\nudge] at (0,2) {${Y}$};
        \node[inout,yshift=\nudge] at (5,2) {${Y}$};
        \node at (2.35,1.1) {${}_{A}$};
        \node at (3.43,1.1) {${}_{S}$};
    \end{tikzpicture}
\end{equation}
Although this is a formal diagram it can be read in a very intuitive way, from left to right: first the agent takes an action determined by the map ${r}$, which gets the initial value of ${X}$ as an input.
Then the environment's state updates as a function of this action and the environment's previous state via the map ${e}$, which generates both a new value of ${Y}$ and a sensor value; and finally the agent's state updates via the map ${u}$, producing a new value of ${X}$. See \citep{biehl_interpreting_2023,baltieri2025imp} for more on string diagram approaches to the sensorimotor loop.
The relationship between string diagrams and Bayesian networks is explored in \citep{fritz2023d-separation,fong2013causal,Jacobs2019causal}.

The reason for defining two types of machine was so that we could couple them in this way, where an action takes place, the environment reacts and a sensor value is received all in the same time step.
This is a modelling choice, however---there is nothing inherently ``agent-like'' about Moore machines or ``environment-like'' about Mealy machines.
We could have made both systems Moore machines at a slight cost in generality, or we could have made the agent a Mealy machine and the environment a Moore machine, with only small changes to what follows.

\subsection{Defining regulation}
\label{regulation.sec}

From now on we fix an agent, in the form of a Moore machine $({X},{r}:{X}\to {A}, {u}:{X}\times {S}\to {X})$ and environment in the form of a Mealy machine $({Y},{e}:{Y}\times{S}\to{Y}\times {A})$.
Given this, we can ask what it means for the agent to perform regulation.

In a classical control theory setup, we would be thinking about the agent regulating the environment.
In that case, we would have a subset ${G}_{Y}\subseteq {Y}$ of the environment's state space called the \emph{good set}, and seek to find a controller (that is, a Moore machine) that, when coupled to the environment, keeps its state inside the good set.
For example, a ``good'' thermostat might be one that keeps its environment's temperature between $19\,{^\circ}\rm C$ and $25\,{^\circ}\rm C$, and the good set would consist of all environment states whose temperature is in this range.

There is a lot that can be said about this classical setup, but our interest here is in embodied biological agents.
Generally speaking, biological agents aren't concerned with regulating their environments but with regulating \emph{themselves}, i.e. their own internal physiological variables \citep{ashby1960design,beer2004autopoiesis}\insFootnote{autopoiesis.footnote}.
One way to set this up would be to have the good set be a subset ${G}_{X}\subseteq {X}$ of the \emph{agent}'s states, e.g. perhaps including a suitable range of body temperatures.
However, it turns out that for our purposes it is not hard to set things up in a more general way that encompasses both the classical setup and the Ashbian one.
This leads to the following simple, albeit slightly counterintuitive, definition:

\defFootnote{autopoiesis.footnote}{From the perspective of autopoiesis they regulate more than this: they maintain their autopoietic organisation {\citep[p.~79]{maturana1980autopoiesis}}. Describing this mathematically is a challenge beyond the scope of this paper, however. 
}

\begin{definition}
    \label{regulation-situation.def}
    A \emph{regulation situation} consists of the agent $(X,r,u)$ and environment $(Y,e)$, together with a subset ${{G}\subseteq {X}\times{Y}}$, which {we refer to} as the \emph{good set}.
\end{definition}

Here we define the good set as a subset of the states of the coupled system, rather than of the agent or the environment alone. This encompasses both the classical setup and the Ashbian one, because given a subset ${G}_{Y}\subseteq{Y}$ of environment states we can take $G$ to be the set of pairs $({x},{y})$ with ${x}\in{X}, {y}\in{G}_{Y}$, whereas given a subset ${G}_{X}\subseteq{X}$ of agent states we can take ${G}$ to consist of pairs $({x},{y})$ with ${x}\in{G}_{X}, {y}\in {Y}$.
We refer to \cref{regulation-situation.def} as a regulation \emph{situation} rather than a regulation \emph{problem} {since} we are not trying to find a controller that regulates a given environment; instead we follow \CAauthor{}'s approach of taking both the controller and the environment as given, along with the good set.

Intuitively, we want to say that the {agent performs regulation} %
if it keeps the state of the coupled system inside the good set.
However, in general, whether this happens or not is a function not just of the dynamics but also of the initial state of the coupled system, i.e. the combined initial state of \emph{both} the agent \emph{and} its environment.
This is illustrated in \cref{schematics.fig}, where trajectories stay inside ${G}$ indefinitely if and only if the initial state is in the region indicated in \cref{phase-space-with-G-and-maximal-R.subfig}.

\newcommand{\figquadrant}[2]{{%
\begin{subfigure}[b]{0.22\textwidth}%
    \includegraphics[width=\textwidth]%
    {images/#1.png}%
    \caption{}%
    \label{#1.subfig}%
\end{subfigure}%
}}

\begin{figure}[tb]%
    \centering%
    \figquadrant{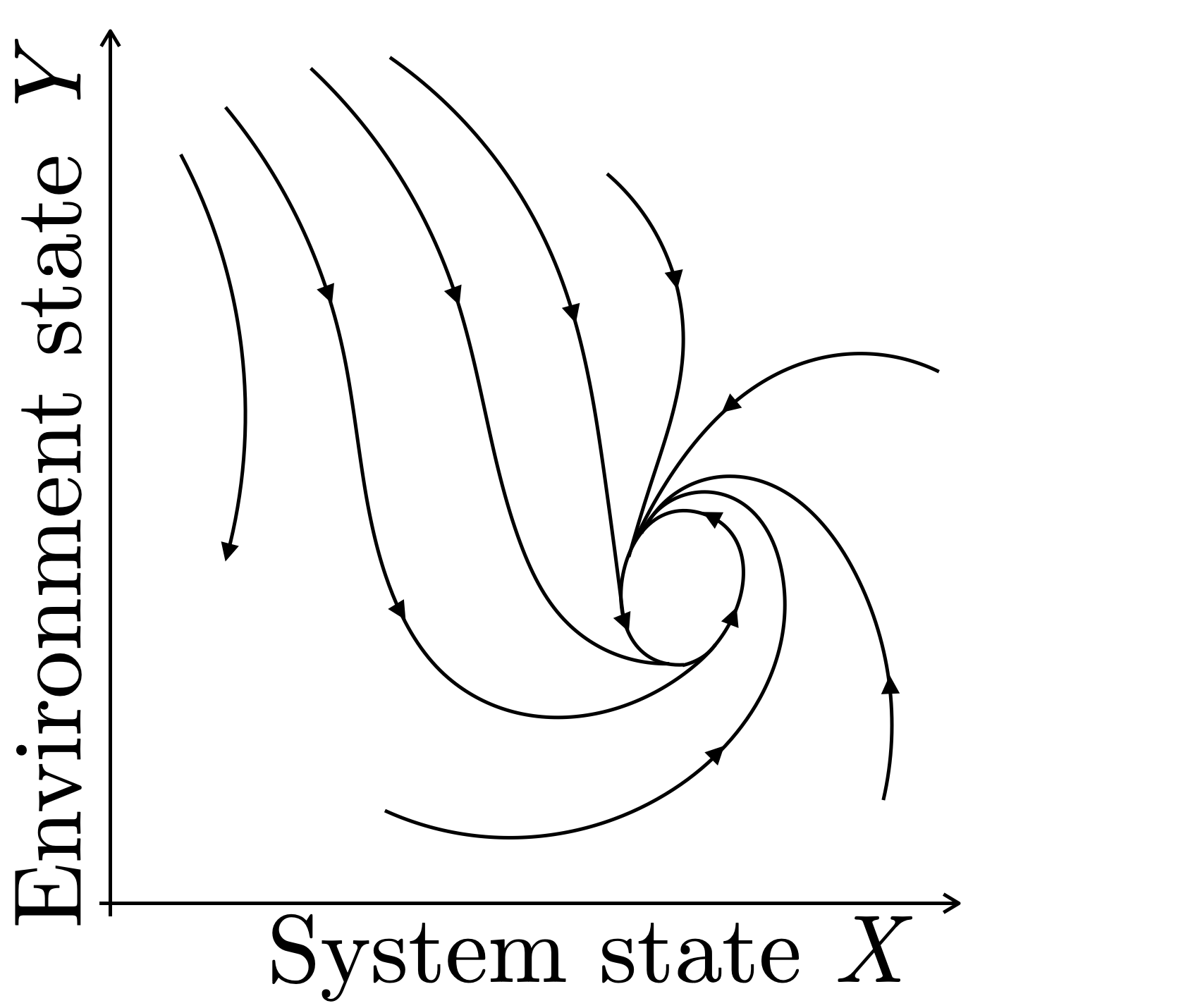}{(a)}%
    ~~~~%
    \figquadrant{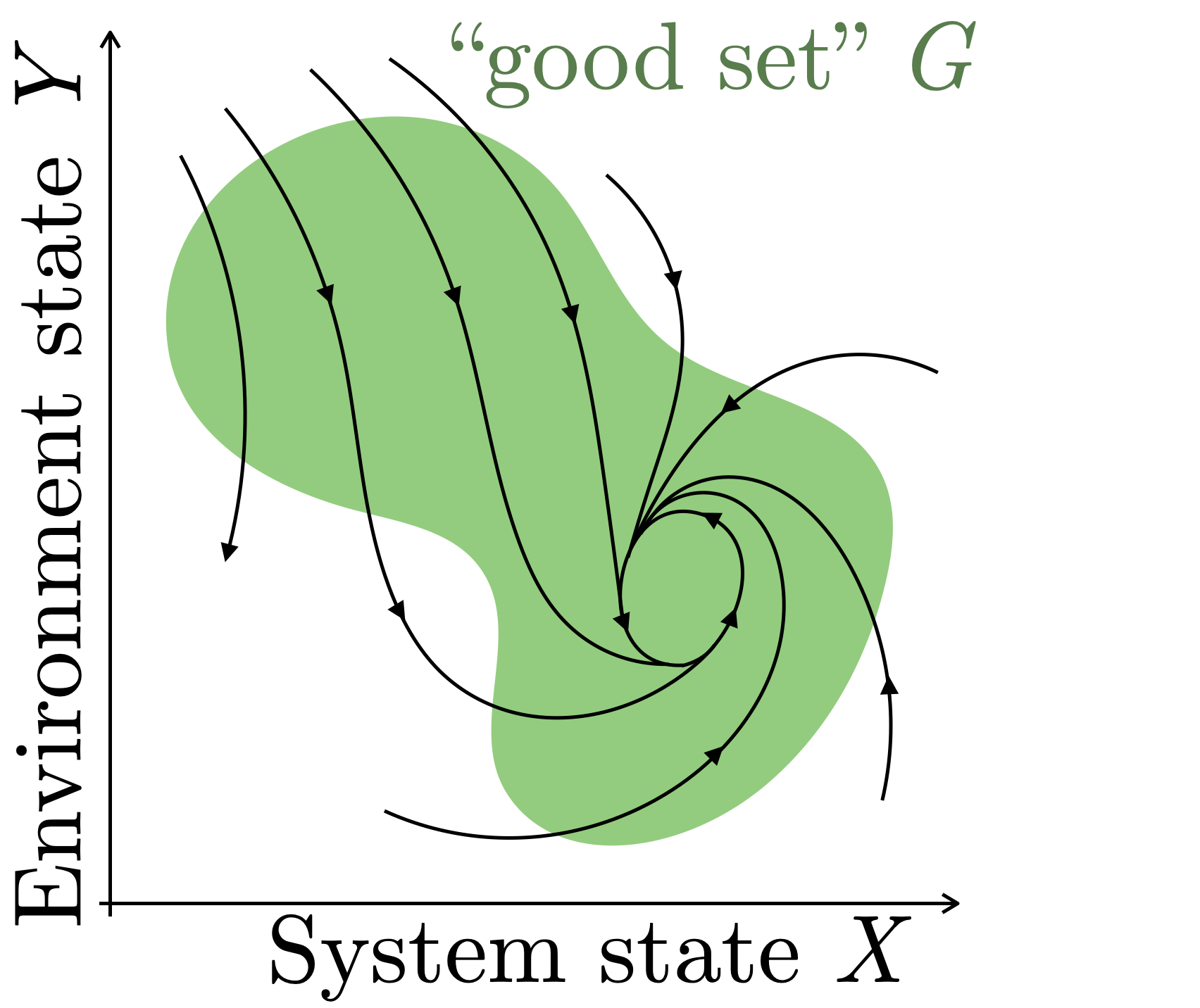}{(b)}%
    \\[1em]%
    \figquadrant{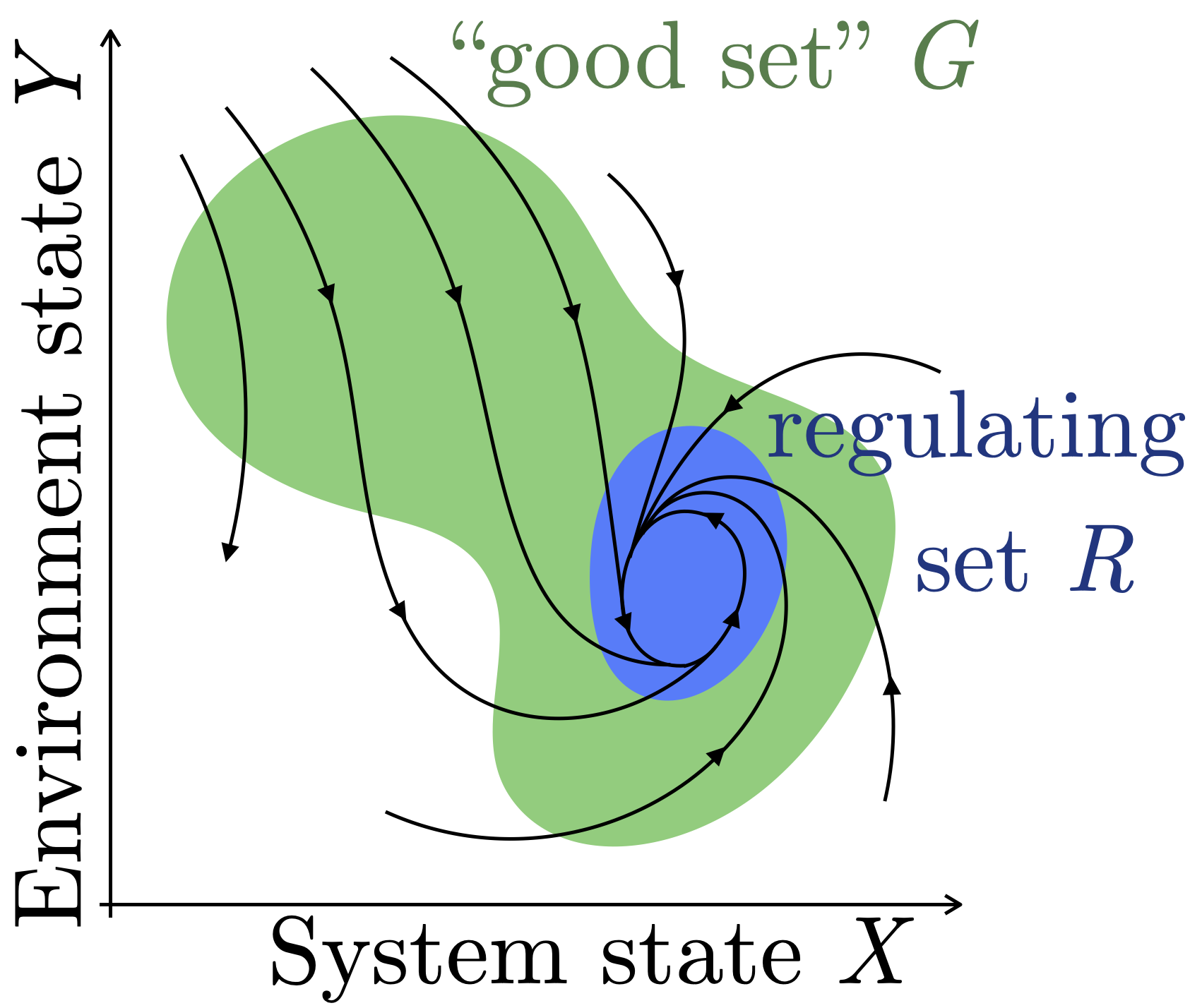}{(c)}%
    ~~~~%
    \figquadrant{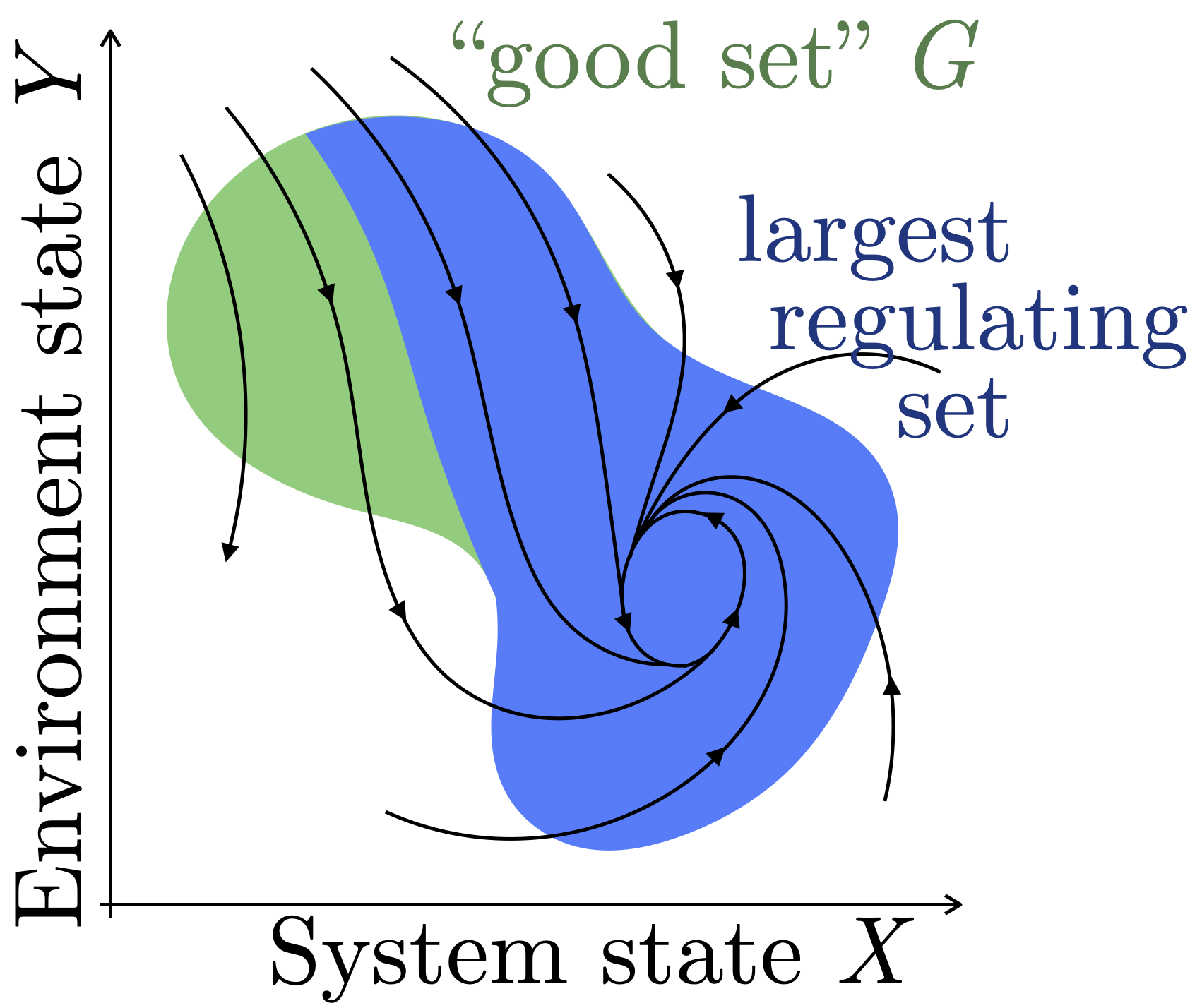}{(d)}%
    \\%
    \caption{\small Schematics drawn in continuous time \citep[cf.][]{Beer1995}, although the definitions in the main text are discrete. (a) A pair of coupled systems together form a dynamical system, some of whose variables belong to each system; in this schematic we assume one variable each, resulting in a two dimensional phase space.
    (b) The same phase space, equipped with a ``good set'' ${G}$.
    Not every trajectory that starts within the good set stays within it indefinitely, but some do.
    (c) A \emph{regulating set} is a nonempty subset of ${R}$ that is forward-closed.
    The existence of a regulating ${R}$ set witnesses that there are some trajectories that stay inside ${G}$ indefinitely.
    (d) There may be many regulating sets; here we show the largest one.
    }%
    \label{schematics.fig}%
\end{figure}

In order to take the state of the coupled system into account, we use the notion of \emph{forward-closed set}.
For our simple notion of dynamical system consisting only of a set ${W}$ and a function ${h}:{W}\to {W}$, a \emph{forward-closed set} is a subset ${V}\subseteq{W}$ such that for each ${v}\in {V}$ we have ${h}({v})\in {V}$.
In other words, to say ${V}$ is forward-closed means that if the state is in ${V}$ on the current time step then it will also be in ${V}$ on the next time step, and hence it will be inside ${V}$ in all future times as well.
An example of a forward-closed set is the basin of attraction of a fixed point, although {not all forward-closed sets are of this form}.
With this notion in hand we can define what it means to perform regulation in our framework.

\begin{definition}
    \label{good-regulator.def}
    Given a regulation situation as in \cref{regulation-situation.def}, we say that the agent is a \emph{good regulator} if we can exhibit a set ${R}\subseteq {X}\times{Y}$ that is (1) non-empty, (2) forward-closed, and (3) contained in ${G}$, i.e.\ ${R}\subseteq {G}$.
    We call such a set ${R}$ a \emph{regulating set}.
\end{definition}

If the initial state is in ${R}$ then the system will remain in ${G}$ indefinitely, so the existence of ${R}$ implies that regulation is possible, given an appropriate initial state.

\myparagraph{Some philosophical context}
The reader might have noticed a strange asymmetry in \cref{good-regulator.def}, and this seems a good point to address this, along with other philosophical issues that have arisen by this point.
The asymmetry is that we say the \emph{agent} is a good regulator, even though \cref{regulation-situation.def,good-regulator.def} are otherwise symmetric between agent and environment (aside from the modelling choice about Moore \emph{vs.}\ Mealy machines).
This is because we are really describing a \emph{stance} one can take toward the agent, of a similar flavour to Dennett's intentional stance.

We think of the agent and its environment as observed physical things (although one or both of them could equally be a model in the mind of an observer as we discuss below).
The two machines and the coupled system specify the dynamics of these systems, which is something physical and in principle empirically testable.
But the good set says something about what \emph{should} happen, and as such we don't expect it to be any kind of physical property of the system itself.
Instead, we think of it as something imposed on the system by the observer: the observer has decided to treat the Moore machine \emph{as if} it is an agent whose goal is described by the good set ${G}$.
This might be because the agent is a robot and the observer its designer, in which case ${G}$ is part of a specification of its desired behaviour, or because the system is an organism and the observer a biologist, in which case the good set could be thought of as a fiction standing in for some knowledge or a hypothesis about its evolutionary history.

We will  explore the consequences of this, from the perspective of an observer who is committed to taking such an `as-if' stance, treating this particular system as an agent with this particular good set ${G}$ and this particular regulating set~${R}$.
We will discuss the agent's `norms' and `beliefs', but it is important to keep in mind that these belong to the observer as much as to the agent itself; they are attributed to the agent by the observer, and a different observer might attribute different norms and beliefs to the same system.

A final philosophical point concerns the role of the regulating set ${R}$.
In general there might be many possible regulating sets, i.e.\ many non-empty forward-closed subsets of ${G}$, but in order to apply our theorem, one must pick one of them.
One can always take the largest regulating set (given by the union of all regulating sets, assuming at least one regulating set exists), but there may be reasons to consider another regulating set besides this one.
For example, proving that some particular set ${R}$ is a regulating set might be much easier than identifying the largest regulating set.
Different choices of ${R}$ will lead to different belief interpretations in \cref{forward-closed-and-interpretations.lemma}.

\section{Models and interpretations}
\label{interpretations.sec}

While \citeauthor{conant1970}'s notion of model is about a map from the system being modelled to the system doing the modelling, our theorem invokes a different kind of model, one that's closely related to so-called \emph{Bayesian filtering interpretations}, though not probabilistic in nature.

To set the stage we give some brief background on Bayesian interpretations.
The core idea is in some ways an old one, with roots going back to \citep{kalman1960new}, but interpretations were introduced as an explicit concept in `as-if' agency in \citep{virgo2021interpreting}, inspired by a category-theoretic approach to conjugate priors \citep{jacobs2020channel}.

To build a Bayesian filtering interpretation, we start with a dynamical system with inputs (which might be an embodied agent but doesn't have to be).
To \emph{interpret} it as performing Bayesian inference about some hidden variable, we regard its internal states as parametrising a prior over the hidden variable.
This means we have a map from the state space of the system to probability distributions over the hidden variable, which we think of as telling us ``what the system's prior is'', as a function of its state.
This must be such that whenever the system's state updates in response to an input, the new prior assigned to the system's new state must equal the posterior that Bayes' theorem mandates.
In other words, the following somewhat informal diagram should commute, where $\Delta({Y})$ is the space of priors over some hidden variable ${Y}$, the agent's state space is ${X}$, and $\psi:{X}\to \Delta({Y})$ is the interpretation map.
\begin{equation}
    \label{belief-updating-informal.eqn}
\begin{tikzcd}
	{\Delta({Y})} &&&&& {\Delta({Y})} \\
	\\
	{{X}} &&&&& {{X}}
	\arrow["{\text{update by Bayes conditioned on \({s}\in {S}\)}}", from=1-1, to=1-6]
	\arrow["\psi", from=3-1, to=1-1]
	\arrow["{\text{dynamical transition on receiving input \({s}\in {S}\)}}", from=3-1, to=3-6]
	\arrow["\psi", from=3-6, to=1-6]
\end{tikzcd}
\end{equation}
Making this formal requires specifying some other data, including the agent's \emph{model}, i.e.\ how the agent believes ${S}$ and ${Y}$ are correlated, and how it believes ${Y}$ changes over time, if ${Y}$ is not assumed to be constant.
Once the details are filled in, \cref{belief-updating-informal.eqn} becomes Equation 5 in \citep{virgo2021interpreting}, the so-called \emph{consistency equation for Bayesian filtering interpretations}.
Our notion of interpretation below has this same general form but differs in its details from this previous work.

Regardless of how the details are filled in, there can be (and generally are) many interpretations of a given system.
It doesn't make sense to ask which interpretation is the correct one, or to the extent that it does, it isn't a question about the system but about something external to it, such as the intention of its designer.
To interpret a system as performing inference requires choosing an interpretation {map}, and this choice belongs to the observer and not to the system itself.

In \citep{biehl_interpreting_2023} this idea was extended to the case of an agent that takes actions on the basis of its Bayesian beliefs, which we will also be concerned with below.
In \citep{baltieri2025imp} the idea is instantiated in the realm of \emph{possibilistic} reasoning, where the agent only cares %
what can or can't happen, as opposed to \emph{probabilistic}, where the possibilities are assigned probabilities.
This is also what we will do, although also in a slightly different way.

Let us now set out the specific notion of interpretation that's relevant for our main theorem, \cref{main-theorem.thm}. We want to consider an agent as having beliefs about its environment, which in our framework means a Moore machine having beliefs about a Mealy machine.

As before, we start with a Moore machine given by $({X}, {{r}:{X}\to {A}},\, {u}:{X}\times{S}\to{X})$ and a Mealy machine which we will write as $({Z},{f}:{Z}\times {A}\to {Z}\times {S})$. 
This Mealy machine is not necessarily the same as the environment $(Y,e)$ mentioned previously.
Instead,
we will think of it as the thing the agent \emph{has beliefs about}.
(Meaning: the thing about which some observer interprets it as having beliefs.)
The picture to have in mind is that the Moore machine might be coupled to an environment $(Y,e)$, but independently of that it might `believe' it is coupled to a different environment $(Z,f)$.
As discussed in \citep{virgo2021interpreting}, 
it might be that $(Y,e)=(Z,f)$,
but it need not be, since we want to allow the possibility that the agent is not in its natural environment, or otherwise has an incorrect model
of how the environment behaves.
All the machines have the same interface, $(S,A)$.

We sometimes refer to the Mealy machine $({Z},{f})$ as the agent's model (as attributed to it by an observer).
This sense of `model' is very different from Conant and Ashby's one.
We sometimes talk about the agent \emph{having} 
 a model of $(Z,f)$,
although to be precise we should say ``the observer interprets the agent as having a model,'' since
the model really belongs as much to the observer as it does to the system itself.

As discussed, to talk about the agent having beliefs about the environment we should define an interpretation map, which should obey certain conditions.
In the present work, this will be a function $\psi:{X}\to\pow({Z})$, where $\pow$ 
means power set.
The idea is that when the agent is in state ${x}\in {X}$, we interpret it as believing that the environment is in one of the states in $\psi({x})$, which is a subset of ${Z}$.
If the set $\psi({x})$ is a singleton $\{{z}\}$, the agent is said to be certain that the environment is in state ${z}$.
If $\psi({x})={Z}$ then we say the agent is completely uncertain about the environment's state; it could be in any state at all.
For anything in between the agent has some partial certainty---it is interpreted as knowing that the state is within some set but not precisely what the state is.
It is also possible that $\psi({x})=\emptyset$, in which case we say the agent has `absurd' beliefs.
This is similar to believing that $0=1$; it signifies an inconsistency in the agent's beliefs, since there is no state of the environment that satisfies them.

\looseness-1
What conditions should we require the interpretation map $\psi:{X}\to\pow({Z})$ to have in order to be valid?
To determine this, let us put ourselves in the mind of someone who knows that the state of the environment is within some set ${B}\subseteq {Z}$.
Suppose that this person then takes some action ${a}\in {A}$ (for the moment we will not worry about how this action was selected) and then receives some sensor value ${s}\in {S}$.
Knowing the values of ${a}$ and ${s}$, this person can then reason as follows:

``I know that initially the environment was in some state ${z}\in {B}$ and that I took action ${a}$, and that means the inputs to the 
 function 
${f}$ must have been $({z},{a})$ for some ${z}\in {B}$.
So, before observing the sensor value ${s}\in {S}$, I knew that the new state and the sensor value together had to be in the set $\{\,({z'},{s})\in {Z}\times {S} \mid \exists\, {z}\in {B}:({z'},{s})={f}({z},{a}) \,\}$.
But in fact I know the value of ${s}$, and so I can deduce that the new state of the environment must be within the set
\newcommand{\update}[3]{\mathsf{update}(#1,#2,#3)}
\begin{multline}
    \label{updated-beliefs.eqn}
    \update{{B}}{{a}}{{s}}\coloneqq \\\{\,{z'} \in {Z} \mid \exists\, {z}\in {B}:{({z'},{s}) = {f}}({z},{a}) \,\},
\end{multline}
where ${s}$ is the known sensor value.''

\Cref{updated-beliefs.eqn} expresses a kind of belief updating, somewhat akin to Bayesian filtering, but in a possibilistic rather than probabilistic framework.
$\update{{B}}{{a}}{{s}}$ is the posterior beliefs about the next environment state, given prior beliefs ${B}$ about the current environment state and data about the action taken and sensor value received, ${a}$ and ${s}$. 

Before turning this into a formal definition of consistency for an interpretation map, we will allow our imaginary person one extra step in their reasoning: they can decide to forget information.
Although the ideal posterior is given by $\update{{B}}{{a}}{{s}}$, we will allow the person to adopt any posterior ${C}$ as long as $\update{{B}}{{a}}{{s}}\subseteq {C}$.
The set ${C}$ can contain less information than $\update{{B}}{{a}}{{s}}$, in the sense that it puts less constraint on what the environment's state might be.
This `forgetting' step might be a pragmatic choice if the more specific information is not needed any more to achieve a task that the person is trying to do.

Returning to our goal of attributing beliefs to a Moore machine, the idea is that as the agent takes actions and receives sensor values, the image of $\psi$ should change over time in the same way that a person's beliefs would, if they were reasoning in the way described above.
This leads to the following definition:

\begin{definition}
    \label{consistent.def}
    Given the agent $({X},{r},{u})$, a model in the form of a Mealy machine $({Z},{f})$ and
    a function $\psi:{X}\to \pow({Z})$, we say $\psi$ is a \emph{consistent belief map} if for every ${x}\in {X}$ and every ${s}\in {S}$ we have $\update{\psi({x})}{{r}({x})}{{s}}\,\subseteq\,\psi({{u}({x},{s})})$.
    That is,
    \begin{multline}
    \label{consistency-equation-GR-version.eqn}
    \{\,{z'} \in {Z} \mid \exists\, {z}\in \psi({x}):({z'},{s}) = {f}({z},{r}({x})) \,\} \\\,\,\subseteq\,\,\psi({{u}({x},{s})}).
    \end{multline}
\end{definition}
\looseness-1
If $\psi$ is a consistent belief map then we  say that $(X,r, u)$ \emph{can be interpreted as having model $(Z,f)$}, since the beliefs $\psi$ assigns to $(X,r,u)$ update in a way consistent with $(Z,f)$.
The idea is that $\psi({x})$ represents the beliefs that the observer attributes to the agent when it is in state ${x}$.
For these beliefs to update consistently over time, the beliefs attributed to ${x'}$ must behave like the posterior beliefs of the person in the informal argument above, i.e.\ they must be a superset of 
$\update{{B}}{{a}}{{s}}$,
where
${B}=\psi({x})$ is the prior beliefs, ${a}={r}({x})$ is the action taken and ${s}$ is the sensor value received.\insFootnote{impossible.footnote}

\defFootnote{impossible.footnote}{It is possible for the left-hand side of \cref{consistency-equation-GR-version.eqn} to be the empty set, which means that $s$ is a ``subjectively impossible'' sensor value, as considered in \citep{virgo2021interpreting}. That is, it can't occur, according to the agent's prior beliefs and model.
In this case \cref{consistency-equation-GR-version.eqn} is always satisfied, so it imposes no constraint on the posterior beliefs.
If ${B}$ is non-empty then there is always at least one ${s}\in {S}$ such that $\update{{B}}{{a}}{{s}}$ is non-empty, i.e.\ there is always at least one subjectively possible sensor value.
}

We now show that there is a correspondence between belief interpretations (that is, consistent belief maps) and forward-closed sets.
Given the coupled system from \cref{coupled-system.def} (which involves the `true' environment $({Y},{e})$) and a forward-closed set ${R}$, we can define a map $\psi:{X}\to\pow({Y})$ by
\begin{equation}
\label{psi-from-R.eqn}
\psi({x})\coloneqq \{\,{y}\in{Y}\mid ({x},{y})\in {R} \,\}.    
\end{equation}
In the following lemma, we show that \cref{psi-from-R.eqn} defines a consistent belief map, where the model $({Z},{f})$ is the same as the true environment $({Y},{e})$.
This is the main technical result needed to prove our ``good regulator theorem'' for embodied agents (the upcoming \cref{main-theorem.thm}).

\begin{lemma}
    \label{forward-closed-and-interpretations.lemma}
    Consider an agent $({X},{r},{u})$ and the environment $({Y},{e})$
    together with a subset ${R}\subseteq {X}\times{Y}$. We have that ${R}$ is forward-closed if and only if $\psi:{X}\to\pow({Y})$ as defined by \cref{psi-from-R.eqn} is a consistent belief map as per \cref{consistent.def}, with model $({{Z},{f}}) = ({Y},{e})$.
\end{lemma}
\begin{proof}
We express the proof as a chain of inferences.
%
%
%
%
\begin{equation*}
\begin{aligned}
    &\text{${R}$ is forward-closed} \\
    \Longleftrightarrow\,\,&\parbox[t][]{3.1in}{
    for every $({x},{y})\in {R}$ we have $({u}({x},{s}),{y'})\in {R}$,\\where $({y}',{s})={e}({y},{r}({x}))$
    } \\
    \Longleftrightarrow\,\,&\parbox[t][]{3.1in}{
    for every ${x}\in{X},y\in\psi({x})$ we have\footnotemark \\${y}'\in\psi({u}({x},{s}))$, where $({y}',{s})={e}({y},{r}({x}))$\\ (by \cref{psi-from-R.eqn})
    } \\
    \Longleftrightarrow\,\,&\parbox[t][]{3.1in}{
    for every ${x}\in{X},y\in\psi({x}),{s}\in{S}$ we have \\$\{\,{y}'\in{Y}\mid ({y}',{s})={e}({y},{r}({x}))\,\}\subseteq\psi(u(x,s))$
    } \\
    \Longleftrightarrow\,\,&\parbox[t][]{3.1in}{
    for every ${x}\in{X},{s}\in{S}$ we have}
\end{aligned} \atop
    \displaystyle{}\bigcup_{{y}\in\psi({x})}\{{y}'\in{Y}\mid ({y}',{s})={e}({y},{r}({x}))\}\subseteq\psi(u(x,s)),
\end{equation*}
\footnotetext{Note when $\psi(x) = \emptyset$, universal quantification over it is always `vacuously' true.}
where the left-hand side of the last inequality 
is equal to
\begin{equation*}
\begin{multlined}[b]
    \{{y}'\in{Y}\mid \exists {y}\in{\psi(x)}:({y}',{s})={e}({y},{r}({x}))\}\\
    = \update{\psi(x)}{r(x)}{s}.\qquad\qquad
\end{multlined}\qedhere   
\end{equation*}
\end{proof}

\Cref{forward-closed-and-interpretations.lemma} establishes a correspondence between forward-closed sets and consistent belief maps.
However, we don't \emph{just} want to interpret the agent as having beliefs about its environment, we also want to interpret it as taking actions that are consistent with its beliefs.
To do that we need to reason about what the agent \emph{wants} (or rather what the observer chooses to interpret it as wanting), and for that reason we introduce a second function, $\phi:{X}\to\pow({Z})$, called the \emph{normative map}.
Unlike the belief map $\psi$, the map~$\phi$ doesn't have to obey any consistency condition.

The idea is that if the agent is in state ${x}\in{X}$ then its goals are satisfied as long as the environment is in one of the states in $\phi({x})\subseteq {Z}$.
If this ever fails to be the case then the agent has failed at its task.
The normative map will end up relating to ${G}$ in the same way that the belief map relates to ${R}$.

It is possible for $\phi({x})$ to be the empty set, which means that the agent never wants to enter state ${x}$, regardless of the environment's state.
It is also possible that $\phi({x})={Z}$, in which case the agent's goal is satisfied when the agent is in state ${x}$, regardless of the state of the environment.

The following definition summarises this ``subjective'' notion of regulation, in which the agent should satisfy its own subjective goals (given by $\phi$) with respect to its own beliefs about its assumed environment (given by $\psi$).
This is subjective in that it's regulation ``from the agent's point of view,'' as attributed to it by an observer.

\begin{definition}
    \label{subjective-good-regulator.def}
    Given the agent $({X},{r},{u})$, a model $({Z},{f})$,
    a consistent belief map $\psi:{X}\to \pow({Z})$ and a normative map $\phi:{X}\to\pow({Z})$, we say the agent is a \emph{subjective good regulator} if $(i)$ we have $\psi({x})\subseteq \phi({x})$ for all ${x}\in{X}$, and $(ii)$ there exists ${x}_0\in {X}$ such that $\psi({x}_0)$ is non-empty.
\end{definition}

Thus, firstly, there should be some state ${x}_0$ in which the agent should have consistent beliefs (i.e. $\psi({x}_0)$ is non-empty) and it should believe that its goal is currently satisfied when it's in that state, i.e.\ $\psi({x}_0)\subseteq\phi({x}_0)$.
If this were not the case the agent would already have failed in its task, and hence would not be a good regulator from its own point of view.
We can assume the agent starts in state $x_0$.

In addition to this, requiring that $\psi$ be a consistent belief interpretation implies that given the action ${a}={r}({x})$ that the agent takes and any sensor value ${s}\in{S}$ that can occur according to the agent's model, its posterior beliefs will also be such that its goal is satisfied.
We can interpret this as saying that the agent will always take an action such that, from its subjective point of view (as attributed by the observer) its goal will still be satisfied in the next time step---and hence also in all future time steps.

There may be states $x$ where the agent's beliefs are inconsistent (i.e.\ such that $\psi(x)=\emptyset$).
This doesn't cause problems because such states can never be entered, as long as the sensor values the agent receives are subjectively possible ones (i.e.\ possible according to its beliefs).
For states with inconsistent beliefs, the condition $\psi(x)\subseteq\phi(x)$ is satisfied automatically.

\subsection{Every good regulator has a model}

As mentioned previously, the interpretation maps $\psi:{X}\to \pow({Z})$ and $\phi:{X}\to \pow({Z})$ are really just the regulating set ${R}\subseteq {X}\times{Y}$ and the good set ${G}\subseteq {X}\times{Y}$ in disguise.
Given ${G}$ and ${R}$, we can set $({Z},{f}) = (Y,e)$ and define $\psi({x}) = \{\,{y}\in{Y}\mid ({x},{y})\in{R}\,\}$ as in \cref{psi-from-R.eqn}, along with
\begin{equation}
    \label{phi-from-G.eqn}
    \phi({x}) \coloneqq \{\,{y}\in{Y}\mid ({x},{y})\in{G}\,\}.
\end{equation}
We claim that with $\phi$ and $\psi$ defined this way, our two definitions of regulation are equivalent, i.e.\ that a Moore machine is a good regulator of a system if and only if it is a subjective good regulator of that system.
If this is so then we have a theorem along the lines of ``every good regulator has a model,'' since being a subjective good regulator involves having beliefs about the environment, which are updated in a similar way to a Bayesian model.

This claim is the subject of our main theorem, whose proof is by now rather straightforward:
\begin{theorem}
    \label{main-theorem.thm}
    Consider the agent $({X},{r},{u})$ and the environment $({Y},{e})$,
    together with sets ${G}\subseteq {X}\times{Y}$ and ${R}\subseteq {X}\times{Y}$.
    Then the  agent is a good regulator 
    in the sense of \cref{good-regulator.def}, with good set ${G}$ and regulating set ${R}$, if and only if it is a subjective good regulator in the sense of \cref{subjective-good-regulator.def}, with model $({Z},{f})=({Y},{e})$, belief map $\psi$ given by \cref{psi-from-R.eqn} and normative map $\phi$ given by \cref{phi-from-G.eqn}.
\end{theorem}
\begin{proof}
    We only need to note that the three components of the two definitions correspond to each other:
    \begin{enumerate}
        \itemsep0pt
        \item ${R}$ is forward-closed if and only if $\psi$ is a consistent belief interpretation. (This was shown in \cref{forward-closed-and-interpretations.lemma}.)
        \item ${R}$ is non-empty if and only if there exists ${x}\in{X}$ such that $\psi({x})$ is non-empty, since ${R}=\{({x},{y})\mid {x}\in {X}, {y}\in \psi({x})\}$.
        \item ${R}\subseteq{G}$ if and only if for every ${x}\in {X}$ we have $\psi({x})\subseteq\phi({x})$. This follows from the definitions of $\psi$ and $\phi$.
        \hfill\qedhere%
    \end{enumerate}%
\end{proof}

Our theorem only says that the agent \emph{can} be interpreted such that the model $({Z},{f})$ is identical to the `true' environment $({Y},{e})$, not that it must be.
It might be that a different interpretation with a simpler model would make more pragmatic sense, as in the example in the following section.

In fact the theorem doesn't require $({Y},{e})$ to be the true environment either.
It only says that if the agent can perform regulation when placed in environment $({Y},{e})$, then it admits an interpretation whose model is $({Y},{e})$. There is nothing to stop $({Y},{e})$ from itself being a counterfactual that the observer is considering, rather than the true environment in which the agent is situated.

\subsection{\dots but some models are trivial}

\citeauthor{conant1970} concluded from their theorem that because every good regulator must have a model, there is no point in trying to design regulators in any way other than first building a model.
Can we draw a similar conclusion from our work?
We think not, and to see why 
we propose the following informal examples.

Consider a doorstop, whose job is to hold doors open when humans desire it.
Despite having no state and no dynamics, the doorstop `achieves its goal' by merely existing, and thus, intuitively, it doesn't seem to make sense to interpret it as an agent with a model.
Nevertheless our theorem says it can be so interpreted; how can we make sense of this?

More formally, consider the case where the state space ${X}$ of the agent and the spaces of sensor values and actions ${S}$ and ${A}$ have only one element each, so that effectively the agent has no state besides the state of existing.
The environment $({Y},{e})$ can still have nontrivial dynamics, and the good set ${G}$ and regulating set ${R}$ can be arbitrarily complicated.

Because ${X}$ has only one element, say ${X}=\{*\}$, we have ${X}\times{Y}\cong {Y}$.
For this reason we can think of ${G}$ and ${R}$ as if they are subsets of ${Y}$ instead of ${X}\times{Y}$.
It's then not hard to see that $\phi(*)$ must essentially just be ${G}$ and similarly $\psi(*)$ is ${R}$.
Because there is only one $*\in{X}$ the beliefs $\psi(*)$ don't change over time.
This is because the consistent updating that our theorem attributes to the agent is of the form ``I believe the system's state is in ${R}$ (hence also in ${G}$) and I know that if it's in ${R}$ now then whatever state it's in on the next time step will also be in ${R}$, so it's consistent to also adopt ${R}$ as my posterior belief.''

This is a trivial interpretation \citep[cf.\ the trivial models of][]{baltieri2025imp} in that although the model $({Z},{f})$ might be complicated and $\phi(*)$ and $\psi(*)$ could be complicated subsets of it, they don't depend in a non-trivial way on the agent's state. Because of this, although our theorem says the doorstop {can} be interpreted this way, it isn't necessarily a good idea to take this interpretation seriously.

This will not always be the case.
If we were to model a detective, who searches for clues, speaks to witnesses etc., eliminating possibilities until they can identify the culprit, then we would expect the model our theorem attributes to genuinely reflect the detective's knowledge about the case, updating in a nontrivial way according to what they find out.
The nontrivial interpretation arises because the detective is highly \emph{intertwined} with their environment, solving their task in a way that relies on complex feedbacks that involve both the detective's internal state and that of their environment.

Many systems of interest will lie between these two extremes, becoming intertwined with their environment to some extent while also relying on its dynamics (including the dynamics of their embodiment within it); such examples are common in ALIFE \citep[e.g.][]{braitenberg1986vehicles,beer2003dynamics}.

We believe there is much more that can be said about intertwinedness and non-triviality. For now we simply conclude that \emph{any} theorem with a statement along the lines of ``every good regulator has a model'' must somehow account for all these examples, and ours does this by allowing such systems to have models that can be more or less trivial.

\section{Conclusions and Discussion}

We have proved a theorem whose statement can be read as ``every good regulator must have a model,'' or more precisely, every good regulator can be interpreted as having a model. 
This notion of admitting an interpretation is more sophisticated than \citeauthor{conant1970}'s notion of ``being a model'' in that it involves the concept of belief updating.
Using this different concept of model allows our theorem to apply to \emph{all} good regulators, without the additional assumptions that Conant and Ashby and their successors have needed to make.

Our theory also has some resemblance to some versions of the free energy principle (FEP).
In particular the versions described in \citep{friston2019free,daCosta2021bayesian} involve a ``synchronisation map'' that has some similarity to our $\psi$, while the overall argument has some qualitative resemblance to our \cref{forward-closed-and-interpretations.lemma}: the synchronisation map arises from a stationary state in a somewhat similar fashion to the way our $\psi$ arises from a forward-closed set.
Indeed, the current work arose in part from an effort to remove some of the approximations and technical assumptions behind \citep{friston2019free} in order to understand what it's really saying.
We don't know whether our result relates to the FEP in a precise way.%

We have emphasised that in order to interpret a system as doing regulation (whether of the objective or subjective kind), an observer must make a series of choices.
These are, roughly, $(i)$ choosing which system to treat as an agent and where to draw its boundary; $(ii)$ choosing what the system is to be interpreted as trying to do; and $(iii)$ choosing \emph{how} it is to be interpreted as achieving its task.
None of these seem to be inherent properties of a system, so both the observer and the system seem to play unavoidable roles.

One might, however, try to eliminate the observer after all. One way to do this is to study \emph{every} possible division of a system into agent and environment, \emph{every} possible goal and \emph{every} valid way to achieve them.
The toolkit of category theory would be well suited to this approach.
Another approach would be to evaluate such choices numerically, giving metrics by which one could be considered better than another. 

Alternatively, one might seek biological arguments to regard some choices as more fundamental than others.
This approach is taken by autopoietic theory and enactivism, which tend to emphasise the importance of metabolic self-maintenance over other goals that might be attributed to a system.
Perhaps this can be formalised, with the good set consisting of all those states in which the system still exists.

We allowed the possibility that the agent can't fully observe its environment; this is the situation that any embodied agent finds itself in.
As a consequence, the model our theorem attributes to the agent includes its uncertainty about the environment's unknown state.
Our result hints that the notion of model can be rehabilitated even within a purely dynamics-based, behavioural approach: models and dynamical coupling appear to be two sides of the same coin.

\section{Acknowledgements}

The authors thank Simon McGregor and Inman Harvey for productive discussions that influenced this work.

This paper was made possible through the support of Grants 62229 (supporting Virgo and Biehl) and 62828 (supporting Biehl) from the John Templeton Foundation. The opinions expressed in this publication are those of the authors and do not necessarily reflect the views of the John Templeton Foundation.

Manuel Baltieri is supported by JST, Moonshot R\&D Grant Number JPMJMS2012.

Matteo Capucci is an Independent Researcher funded by the Advanced Research + Innovation Agency (ARIA).

\footnotesize
\bibliographystyle{hapalike}
\bibliography{refs}

\end{document}